\documentclass[journal]{IEEEtran}
\usepackage{amsmath,amsfonts}
\usepackage{algorithmic}
\usepackage{algorithm}
\usepackage{array}
\usepackage[caption=false,font=normalsize,labelfont=sf,textfont=sf]{subfig}
\usepackage{textcomp}
\usepackage{stfloats}
\usepackage{url}
\usepackage{verbatim}
\usepackage{graphicx}
\usepackage{cite}
\usepackage{amsthm}   
\usepackage{amsmath}  
\usepackage{amssymb}  
\usepackage{amsfonts}
\usepackage{booktabs}

\newtheorem{theorem}{Theorem}

\newtheorem{lemma}{Lemma}
\newtheorem{corollary}{Corollary}

\newtheorem{assumption}{Assumption}

\usepackage{xcolor}

\newcommand*{\Scale}[2][4]{\scalebox{#1}{$#2$}}
\newcommand{\indep}{\perp \!\!\! \perp}
\newcommand{\notindep}{\not \! \perp \!\!\! \perp}

\begin{document}

\title{Generalized Encouragement-Based Instrumental Variables for Counterfactual Regression}

\author{Anpeng~Wu, Kun~Kuang, Ruoxuan~Xiong, Xiangwei~Chen, Zexu~Sun, \\ Fei~Wu,~\IEEEmembership{Senior~Member,~IEEE}, Kun~Zhang,~\IEEEmembership{Member,~IEEE}
\thanks{\IEEEauthorrefmark{1}{Corresponding authors.}}
\IEEEcompsocitemizethanks{
\IEEEcompsocthanksitem Anpeng Wu, Kun Kuang, Xiangwei Chen and Fei Wu are with the College of Computer Science and Technology, Zhejiang University, China.
\protect
(E-mail: anpwu@zju.edu.cn; kunkuang@zju.edu.cn; 3220102218@zju.edu.cn; wufei@zju.edu.cn).
\IEEEcompsocthanksitem Ruoxuan Xiong is with the Department of Quantitative Theory and Methods, Emory University, Atlanta, USA
\protect
(E-mail: ruoxuan.xiong@emory.edu).
\IEEEcompsocthanksitem Zexu Sun is with Gaoling School of Artificial Intelligence, Renmin University of China, China.
\protect
(E-mail: sunzexu21@ruc.edu.cn).
\IEEEcompsocthanksitem Kun Zhang is with Carnegie Mellon University, USA.
\protect
(E-mail: kunz1@cmu.edu).
}
\thanks{This work has been submitted to the IEEE for possible publication. Copyright may be transferred without notice, after which this version may no longer be accessible.}
}

\markboth{Journal of \LaTeX\ Class Files}%
{Shell \MakeLowercase{\textit{et al.}}: A Sample Article Using IEEEtran.cls for IEEE Journals}


\maketitle

\begin{abstract}
In causal inference, encouragement designs (EDs) are widely used to analyze causal effects, when randomized controlled trials (RCTs) are impractical or compliance to treatment cannot be perfectly enforced. Unlike RCTs, which directly allocate treatments, EDs randomly assign encouragement policies that positively motivate individuals to engage in a specific treatment. These random encouragements act as instrumental variables (IVs), facilitating the identification of causal effects through leveraging exogenous perturbations in discrete treatment scenarios. However, real-world applications of encouragement designs often face challenges such as incomplete randomization, limited experimental data, and significantly fewer encouragements compared to treatments, hindering precise causal effect estimation. To address this, this paper introduces novel theories and algorithms for identifying the Conditional Average Treatment Effect (CATE) using variations in encouragement. Further, by leveraging both observational and encouragement data, we propose a generalized IV estimator, named \textbf{En}couragement-based \textbf{Counte}rfactual \textbf{R}egression (\textbf{EnCounteR}), to effectively estimate the causal effects. Extensive experiments on both synthetic and real-world datasets demonstrate the superiority of EnCounteR over existing methods.
\end{abstract}


\section{Introduction}

\begin{figure}[th]
\begin{center}
\includegraphics[width=0.97\linewidth]{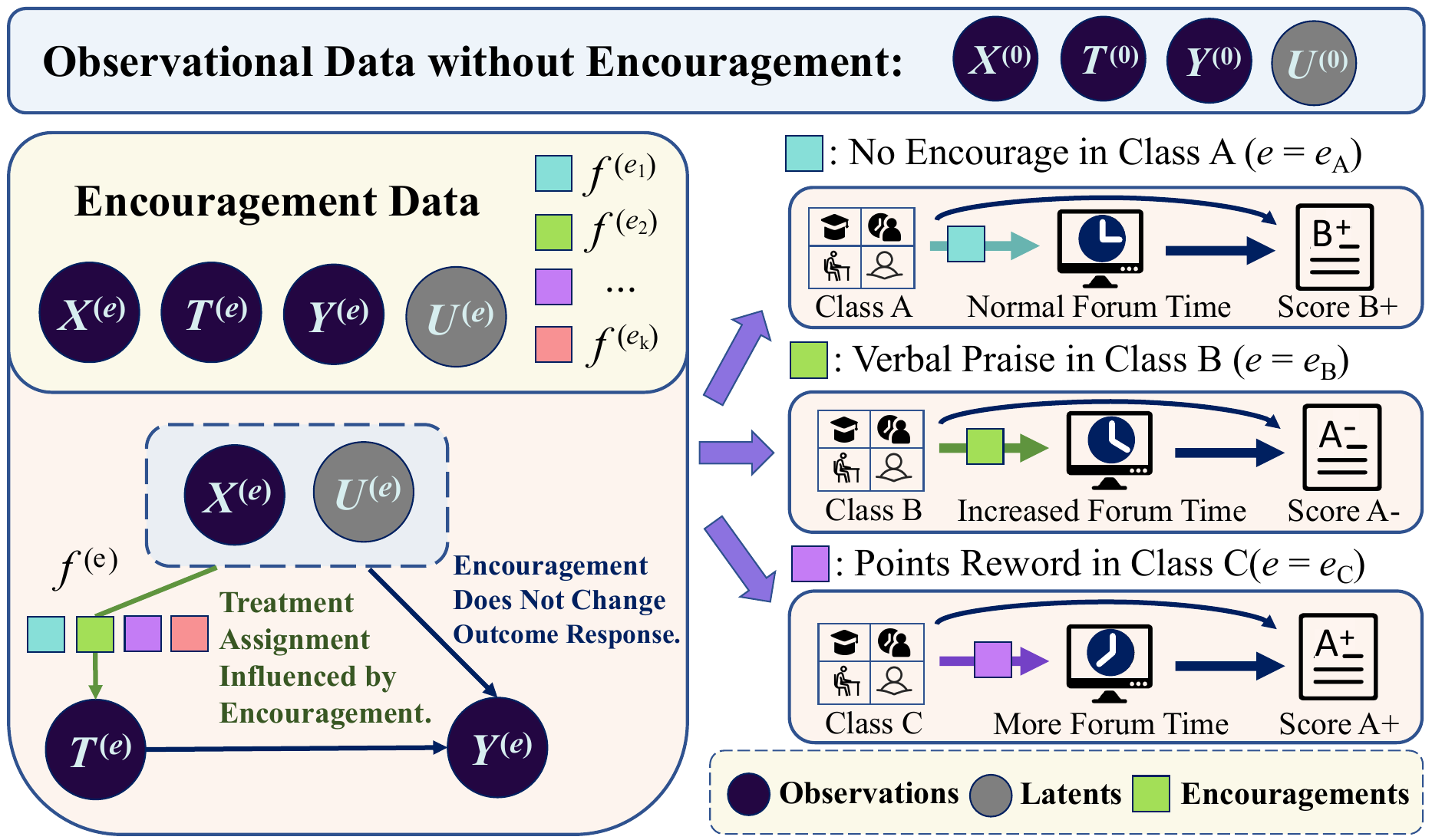}
\end{center}
\caption{Overview of the Encouragement Design Framework. For example, in online course platforms like Coursera, edX, and Udacity, using only observational data to control observed confounders $X^{(0)}$, we can not consistently estimate the causal effects of forum engagement duration $T^{(0)}$ on exam scores $Y^{(0)}$ due to the presence of unmeasured confounders $U^{(0)}$. Therefore, we use varied encouragement policies (Class A: ${e}_\text{A}=\text{None}$, Class B: ${e}_\text{B}=\text{Praise}$, Class C: ${e}_\text{C}=\text{Points}$) to encourage longer forum engagement duration (treatments $T^{(e)}$), while these policies do not have a direct effect on exam scores (outcomes $Y^{(e)}$), which offers opportunities to identify causal effects. 
}
\label{fig:abs}
\end{figure}

\IEEEPARstart{C}{ausal} inference is a powerful statistical modeling tool for explanatory analysis and plays a crucial role in fields like healthcare, economics, and social sciences \cite{angrist1996identification,imbens2015causal,burgess2017review,shalit2017estimating,yao2021survey, wooldridge2016introductory, devriendt2020learning, kuang2020data}. 
While Randomized Controlled Trials (RCTs) are the gold standard for analyzing causal relationships in the presence of unmeasured confounders, randomized treatment allocation often suffers from noncompliance and ethical issues \cite{ kohavi2011unexpected,bottou2013counterfactual,wu2022learning}.
Therefore, randomized encouragement designs (EDs), which randomly assign encouragement policies that positively motivate individuals to engage in a specific treatment, are becoming more popular \cite{small2008war,west2008alternatives} and are widely used to study the causal effects \cite{holland1988causal,small2007sensitivity,kang2016peer}. 

For example, \cite{sexton1984clinical, permutt1989simultaneous} randomly encouraged physicians to advise against smoking to study the effect of smoking on birth weight in pregnant mothers; \cite{angrist1996identification,bang2007estimating,kang2016peer} employed random intent-to-treat to encourage treatment adoption for addressing non-compliance issues. As shown in Figure \ref{fig:abs}, these random encouragements serve as instrumental variables (IVs), which only positively motivate the choice of treatment, while the outcome response remains unaffected by encouragements. However, these discrete encouragement-based IVs are limited to identifying the local average treatment effect in discrete treatments under the assumption of monotonicity \cite{angrist1996identification,pearl2010causal,wooldridge2016introductory}. When the number of encouragements is significantly fewer than the treatment choices, it results in weak and sparse treatment variation from the encouragements.

Consider a generalized education scenario depicted in Figure \ref{fig:abs}, where observational data cannot consistently estimate the causal effects due to the unmeasured confounders. Thus, on online course platforms such as Coursera, edX, and Udacity \cite{breslow2013studying,reich2015rebooting,anderson2014engaging,kizilcec2014encouraging}, we adopt various encouragement policies (${e}_\text{A}=\text{None}$, ${e}_\text{B}=\text{Praise}$, ${e}_\text{C}=\text{Points}$) to motivate longer forum engagement (i.e., treatments $T$), which changes the distribution of $T$ given $X$, in other words, increases time spent on the forum to varying encouragements.
Then we can treat random encouragements as IVs and leverage the exogenous perturbations to identify treatment effects on exam scores (i.e., outcomes $Y$). 
However, in many real applications, practitioners typically apply nonrandom encouragement to existing classes, and the number of encouragements may far less than the treatment itself, leading to endogeneity issues and hindering precise causal effect estimation. This restricts its applicability in existing nonrandom social science experiments.
To this end, under nonrandom encouragements and continuous treatments, we develop a novel theory for identifying the Conditional Average Treatment Effect and propose algorithms with identifiability guarantees to estimate the causal effects.

Overall, in many existing nonrandom social experiments, conventional encouragement designs suffer several issues, 
incomplete randomization, limited experimental data, and significantly fewer encouragement policies compared to continuous treatments, resulting in unreliable causal estimations. Furthermore, due to unmeasured confounders, large observational datasets also fail to identify causal effects reliably. Therefore, to improve experimental efficiency, we treat the large observational data ${X^{(0)}, T^{(0)}, Y^{(0)}}$ as a form of special encouragement with ${e} = {e}_\text{0}$, utilizing it as additional information to identify the CATE through small scale encouragement experiments, thereby reducing experimental costs. By leveraging both observational and encouragement data, we propose a generalized IV estimator, named \textbf{En}couragement-based \textbf{Counte}rfactual \textbf{R}egression (\textbf{EnCounteR}\footnote{The codes are available at: \url{https://github.com/anpwu/EnCounteR/}}), to effectively estimate causal effect. Notably, conventional two-stage IV Regression is just a specific instance of our EnCounteR. We provide novel theory and algorithms with identifiability guarantees to relax conventional IV requirements and achieve more precise causal effects estimation. 
Empirical evaluations demonstrate the superiority of EnCounteR.

\section{Related Work}

\textbf{Encouragement designs} have been widely used for analyzing
causal effects, when RCTs are impractical or compliance to treatment cannot be perfectly enforced \cite{sexton1984clinical,permutt1989simultaneous,angrist1996identification,bang2007estimating}.
In the social sciences, where random treatment assignments may be expensive, harmful, or unethical, \cite{angrist1996identification,hirano2000assessing,bang2007estimating,kang2016peer} employed random intent-to-treat as instruments to encourage treatment for addressing non-compliance issues. \cite{fletcher2010social,an2015instrumental,kang2016peer} utilized personalized encouragement assumptions to study the peer effect in school settings. 
\emph{However, real applications of encouragement designs often pose challenges, including non-randomized encouragements, limited experimental data, and a smaller number of encouragements compared to continuous treatments, hindering precise causal effect estimation.}

\textbf{Instrumental variables} induce exogenous perturbations to treatment variable, allowing for the estimation of causal effects in the presence of unmeasured confounders \cite{hirano2000assessing, sovey2011instrumental,an2015instrumental,kang2016peer,cheng2023discovering,sun2024sequential,zhao2024networked}. 
Traditional IV two-stage regression first identifies treatment variation caused by IVs, then uses it to estimate the dependent variable in the second stage \cite{wald1940fitting,angrist1995identification,angrist1996identification}. Based on the sieve theories \cite{newey2003instrumental}, researchers have developed numerous non-linear IV variants \cite{hartford2017deep,singh2019kernel,muandet2020dual,2019deepgmm,lewis2020agmm,xu2020dfiv,wu2022instrumental}. 
While in continuous treatments with discrete IVs, these methods are prone to have a high variance due to limited exogeneity variation, hindering precise causal effect estimation. \emph{To this end, we develop novel identification theory and algorithms utilizing the variation of encouragements.}
 
Recently, there have also been some works studying invariant learning across \textbf{multiple environments} \cite{arjovsky2019invariant,duchi2021learning,creager2021environment,liu2021heterogeneous,liu2021kernelized,wang2023out}.
\cite{arjovsky2019invariant} identified causally invariant relationships in different environments, assuming their existence for exploration.
\cite{liu2021heterogeneous,liu2021kernelized} generated environments and proposed a maximal invariant predictor, integrating environment inference with invariant learning to improve prediction. These studies can effectively identify causal variables and estimate the total effect of treatments and confounders on outcomes \emph{but fail to identify the causal effect of the treatments on outcomes}.

\section{Problem Setup and Solutions}
\subsection{Notations}
Following \cite{liu2021heterogeneous,liu2021kernelized}, we consider a dataset $\mathcal{D} = \{ \mathcal{D}^{(e_k)} \}_{e_k \in \mathcal{E}}$, which comprises multiple datasets $\mathcal{D}^{(e_k)} = \{x_i^{(e_k)}, t_i^{(e_k)}, y_i^{(e_k)} 
\mid u_i^{(e_k)} \}_{i=1}^{n_k}$ under different encouragement designs $e_k$ in  $\mathcal{E} = \{e_0,e_1,\cdots, e_K\}$, and $n_k$ is sample size in encouragement $e_k$. 
Within each dataset $\mathcal{D}^{(e_k)}$, the variables $x_i^{(e_k)} \in \mathcal{X}$ and $u_i^{(e_k)} \in \mathcal{U}$ are respectively the observable and unmeasured confounders, potentially confounding the analysis of the causal effect of the treatment variables $t_i^{(e_k)} \in \mathcal{T}$ on the outcome variables $y_i^{(e_k)} \in \mathcal{Y}$.
As illustrated in Figure \ref{fig:abs}, observational data alone cannot identify the Conditional Average Treatment Effects (CATE) due to unmeasured confounders. Therefore, we apply $K$ different encouragement policies to promote treatment adoption without directly manipulating the treatment in certain candidate groups
$\mathcal{D}^{(e_k)} = \{x_i^{(e_k)}, t_i^{(e_k)}, y_i^{(e_k)} 
\mid u_i^{(e_k)} \}_{i=1}^{n_k}$:
\begin{eqnarray}
    \label{eq:T}
    t_i^{(e_k)} & = &
    f^{(e_k)}_{\Phi}(x_i^{(e_k)}, u_i^{(e_k)}), \\ 
    y_i^{(e_k)} & = & g_{\Psi}( t_i^{(e_k)}, x_i^{(e_k)}) + \varepsilon( u_i^{(e_k)} ),
    \label{eq:Y}
\end{eqnarray}
where  $f^{(e_k)}_\Phi(\cdot)$ denotes different treatment assignment mechanisms, with unknown parameters $\Phi^{(e_k)}$ for various encouragements $e_k \in \mathcal{E}$, $g_{\Psi}(\cdot)$ depicts the heterogeneous treatment effect with unknown parameters $\Psi$, and $\varepsilon( \cdot )$ embeds the unmeasured confounding effects from $u_i^{(e_k)}$ as additive noise that is a common assumption used in causality \cite{newey2003instrumental,imbens2015causal,hartford2017deep}.
In our designs, the large observational data $\mathcal{D}^{(e_0)}$ can optionally serve as a special dataset with no-encouragement $e = e_0$, to increase the number of encouragements and samples in the overall dataset $\mathcal{D} = { \mathcal{D}^{(e_k)} }_{e_k \in \mathcal{E}}$.

Traditional works assign random encouragements that are exogenous and independent of $U$, i.e., $U \indep \mathcal{E}$, for analyzing causal effects. In practice, however, practitioners typically apply nonrandom encouragement to existing groups like classes or cities, making the estimates unreliable. For instance, when encouraging forum participation in classes led by advanced teachers, it's challenging to identify whether improved exam scores are due to encouragement or teachers. Thus, we collect numerous proxies $X$ for $U$ and propose a covariate balancing module to address the distribution shift from encouragements, thereby relaxing the assumption to conditional independence.

To simplify notation, we denote random variables as uppercase notation, $X = \{X^{(e_k)}\}_{e \in \mathcal{E}}$, where $X^{(e_k)} = \{x_i^{(e_k)}\}_{i=1}^{n^{(e_k)}}$ signifies the sample vector of observed pre-treatment variables for each encouragement design $e_k$. Similarly, we define the vectors $U^{(e_k)}$, $T^{(e_k)}$, and $Y^{(e_k)}$, each corresponding to the respective encouragement design denoted by $e$.
Furthermore, we use $\mathbb{E}[X^{(e_k)}]$ and ${\operatorname{Var}\left(X^{(e_k)}\right)}$ to denote the expected value and variance of $X^{(e_k)}$, respectively, and ${\operatorname{Cov}\left(Y^{(e_k)}, X^{(e_k)}\right)}$ to represent the covariance between $Y^{(e_k)}$ and $X^{(e_k)}$ in $e_k$.

\subsection{Assumptions and Theorems}

In many real applications, encouragement designs suffer from nonrandom encouragements, limited samples, and sparse encouragement policies, resulting in unreliable causal estimations.
To address this,
we leverage both observational and encouragement data $\mathcal{D} = \{ \mathcal{D}^{(e_k)} \}_{e_k \in \mathcal{E}}$ and develop novel theory and algorithms to identify causal effects on outcomes. To this end, we naturally start with a linear setting to build intuition on the necessary assumptions and corresponding theorems, and then put efforts into generalizing these insights to more complex nonlinear settings.

\subsubsection{Formalization in Linear Setting}

For illustration, consider a linear reformulation of Eq. \eqref{eq:Y}:
\begin{eqnarray}
    y_i^{(e_k)} = \psi_t t_i^{(e_k)} + \psi_x x_i^{(e_k)} +\psi_u u_i^{(e_k)}.
    \label{eq:linearY}
\end{eqnarray}
where the coefficient $\psi_t$ is the constant causal effect of interest, and the treatment assignments $t_i^{(e_k)} = f^{(e_k)}_{\Phi}(x_i^{(e_k)}, u_i^{(e_k)})$ can be arbitrary functions across encouragements $e_k$.
Under Assumption \ref{ass:linear} and \ref{ass:indep}, we propose a novel identification theorem of causal effect $\psi_t$. 

\begin{assumption}[Linearity] \label{ass:linear}
The outcome variable $Y$ is a linear function of variables $T$,$X$, and $U$.
\end{assumption}

\begin{assumption}[Independence] \label{ass:indep}
    $X$ and $U$ are independent of the encouragements, i.e., $\{X, U\} \indep \mathcal{E}$.
\end{assumption}

This assumption arises from the common linear case where $X \indep \mathcal{E}$ and $U \indep \mathcal{E} \mid X$. In non-linear settings, we retain only $U \indep \mathcal{E} \mid X$ for greater flexibility.

\begin{theorem}
    \label{theorem:1}
    Under Assumptions \ref{ass:linear} \& \ref{ass:indep}, given two datasets $\{ \mathcal{D}^{(e_0)}, \mathcal{D}^{(e_1)}\}$ with different encouragements $\{e_0,e_1\} \in \mathcal{E}$, the causal effect $\psi_t$ is identifiable. 
\end{theorem}

\begin{proof}
    First, define $\beta_{Y|X}^{(e_k)} = \frac{\operatorname{Cov}\left(Y^{(e_k)}, X^{(e_k)}\right)}{\operatorname{Var}\left(X^{(e_k)}\right)}$, $\beta_{T|X}^{(e_k)} = \frac{\operatorname{Cov}\left(T^{(e_k)}, X^{(e_k)}\right)}{\operatorname{Var}\left(X^{(e_k)}\right)}$, and $\beta_{U|X}^{(e_k)} = \frac{\operatorname{Cov}\left(U^{(e_k)}, X^{(e_k)}\right)}{\operatorname{Var}\left(X^{(e_k)}\right)}$ for any $e_k$.

    Then, we can reformulate $\beta_{Y|X}^{(e_k)}$ as follows:
    \begin{equation}
    \begin{aligned}
    & \Scale[1.0]{{\beta_{Y|X}^{(e_k)} = \frac{\operatorname{Cov}\left(Y^{(e_k)}, X^{(e_k)}\right)}{\operatorname{Var}\left(X^{(e_k)}\right)}}}   \\ 
    & =  \Scale[1.0]{\frac{{\operatorname{Cov}\left(\psi_t T^{(e_k)}+\psi_x X^{(e_k)}+\psi_u U^{(e_k)}, X^{(e_k)} \right)}}{{\operatorname{Var}\left(X^{(e_k)}\right)}}}  \\  
    & =  \Scale[1.0]{ 
    \psi_t \beta_{T|X}^{(e_k)} 
    + \psi_x+\psi_u \beta_{U|X}^{(e_k)}}.
    \end{aligned}
    \label{eq:4}
    \end{equation}
  
    Given $\{X, U\} \indep \mathcal{E}$, the covariance between $U^{(e_k)}$ and $X^{(e_k)}$, and hence $\beta_{U|X}^{(e_k)}$ remains constant across encouragements. Then, we define $b = \psi_x+\psi_u \Scale[0.93]{\beta_{U|X}^{(e_k)}}$. For encouragements $e_0$ and $e_1$, we have $\beta_{Y|X}^{(e_0)} = \psi_t \beta_{T|X}^{(e_0)} + b$ and
    $\beta_{Y|X}^{(e_1)} = \psi_t \beta_{T|X}^{(e_1)} + b$. Then, the causal effect is identified by: $\psi_t = \frac{\beta_{Y|X}^{(e_1)} - \beta_{Y|X}^{(e_0)}}{\beta_{T|X}^{(e_1)} - \beta_{T|X}^{(e_0)}}$, and thus the causal effect $\psi_t$ is identifiable. 
\end{proof}

We consider the scenario where $X$ are exogenous variables, which are independent of the unmeasured confounders $ U $, i.e., $ X \indep U $, and do not directly affect the outcome $Y$, i.e., $\psi_x=0$. Under these conditions, the traditional IV regression could be viewed as a specific instance of our theorems.

\begin{lemma}
\label{lemma:1}
    With $\psi_x=0$ and $X \indep U$ in linear model Eq. \eqref{eq:linearY}, given the observations $\mathcal{D}^{(e_0)}$, then $\psi_t$ is identifiable by $\psi_t = \beta_{Y|X}^{(e_0)} / \beta_{T|X}^{(e_0)} = \text{Cov}(Y, X)/\text{Cov}(T, X)$.
\end{lemma}

As seen in previous studies \cite{newey2003instrumental,pearl2009causality,imbens2015causal}, while we can use encouragements $e \in \mathcal{E}$ as IVs, the number of encouragements may be considerably fewer than the potential treatment values, resulting in high estimation variance. Our theorem surpasses this, identifying causal effects in continuous cases with lower variance.

\subsubsection{GMM Reformulation}

Theorem \ref{theorem:1} provides a linear analytical solution (LAS) for datasets with two encouragements. However, when the number of encouragements exceeds two, the system becomes over-identified, characterized by more equations than unknowns. To resolve this, we reformulate the issue as a generalized method of moments (GMM) problem. We use a residual $\epsilon_i^{(e_k)}$ to identify the parameters ${\psi_t, b}$, where the residual is defined as:
$
\epsilon_i^{(e_k)} =  y_i^{(e_k)} - \psi_t t_i^{(e_k)} - b x_i^{(e_k)} 
 =  
\psi_u (u_i^{(e_k)} - \beta_{U|X}^{(e_k)} x_i^{(e_k)})$. Here, $\epsilon_i^{(e_k)}$ is not merely $u_i^{(e_k)}$ but a transformation meant to eliminate correlation with the observed covariate $x_i^{(e_k)}$.

\begin{theorem}\label{theorem:2}
Under Linearity Assumption \ref{ass:linear}, \\
\begin{equation}
    \begin{aligned}
    &  \Scale[1.0]{\frac{1}{\psi_u} \operatorname{Cov}(\epsilon^{(e_k)} , X^{(e_k)} )}
    \\ & = 
    \resizebox{0.84\linewidth}{!}{$\operatorname{Cov}\left(U^{(e_k)}, X^{(e_k)}\right) - \frac{\operatorname{Cov}\left(U^{(e_k)}, X^{(e_k)}\right)}{\operatorname{Var}\left(X^{(e_k)}\right)}\operatorname{Var}\left(X^{(e_k)}\right)$}.
     \end{aligned}
    \end{equation} 
Accordingly, $\operatorname{Cov}(\epsilon^{(e_k)} , X^{(e_k)} ) = 0$ for any $e_k \in \mathcal{E}$.
\end{theorem} 

Define $\tilde{\epsilon}^{(e_k)} = \epsilon^{(e_k)}-\mathbb{E}[\epsilon^{(e_k)}] $ = $\tilde{Y}^{(e_k)} - \psi_t \tilde{T}^{(e_k)} - b \tilde{X}^{(e_k)} $, where $\{\tilde{T},\tilde{X},\tilde{Y}\}$ are de-meaned variables. Based on Theorem \ref{theorem:2}, i.e., $ \mathbb{E}[\tilde{\epsilon}^{(e_k)}\tilde{X}^{(e_k)}] = \operatorname{Cov}(\epsilon^{(e_k)} , X^{(e_k)} ) = 0$, we can derive $K+1$ moments for $\mathcal{E}=\{e_0,e_1,\cdots,e_K\}$:
\begin{eqnarray}
\Scale[0.9]{
g_X\left(\psi_t, b\right)=\left[\begin{array}{c}
\mathbb{E}[(\tilde{Y}^{(e_0)} - \psi_t \tilde{T}^{(e_0)} - b \tilde{X}^{(e_0)})\tilde{X}^{(e_0)}] \\
\cdots \\
\mathbb{E}[(\tilde{Y}^{(e_K)} - \psi_t \tilde{T}^{(e_K)} - b \tilde{X}^{(e_K)})\tilde{X}^{(e_K)}]
\end{array}\right]
}.
\end{eqnarray}
Since the function $Y^{(e_k)} - \hat{Y}^{(e_k)}_{\psi_t,b}$ is only related to $U$ and $X$, where $\hat{Y}^{(e_k)}_{\psi_t,b} = \hat{\psi}_t T^{(e_k)}+\hat{b}X^{(e_k)}$, under Independence Assumption \ref{ass:indep}, we can conclude that $[Y^{(e_k)} - \hat{Y}^{(e_k)}_{\psi_t,b}] \indep \mathcal{E}$:
\begin{eqnarray}
\Scale[0.88]{
g_{\mathcal{E}}\left(\psi_t,b\right) = 
\left[\begin{array}{c}
\mathbb{E}[Y^{(e_i)} - \hat{Y}^{(e_i)}_{\psi_t,b}]  -
\mathbb{E}[Y^{(e_j)} - \hat{Y}^{(e_j)}_{\psi_t,b}] 
\\
\operatorname{Var}[Y^{(e_i)} - \hat{Y}^{(e_i)}_{\psi_t,b}]  -
\operatorname{Var}[Y^{(e_j)} - \hat{Y}^{(e_j)}_{\psi_t,b}]
\end{array}\right]_{i \not= j}}, 
\end{eqnarray}
and the GMM estimator can be written as:
\begin{eqnarray} \label{eq:gmm}
\left({\psi}_t^*, {b}^* \right)
= 
\arg \min_{\hat{\psi}_t, \hat{b}} 
\left[ 
g_X^\prime \cdot W_X \cdot g_X
+ 
g_{\mathcal{E}}^\prime \cdot W_{\mathcal{E}} \cdot g_{\mathcal{E}}
\right],
\end{eqnarray}
where $\{W_X,W_{\mathcal{E}}\}$ are non-negative definite matrices, the optimal $W^\ast$ depends on the moments covariance matrix.

\subsubsection{Generalization to non-linear settings}

Recall the generalized non-linear settings in Eq. \eqref{eq:Y}, i.e., $y_i^{(e_k)} = g_{\Psi}( t_i^{(e_k)}, x_i^{(e_k)}) + \varepsilon( u_i^{(e_k)} )$,
where the outcome response function $g_{\Psi}(\cdot)$ and noise $\varepsilon( \cdot )$ remain constant across different encouragements. Under the classical IV assumptions \cite{newey2003instrumental,hartford2017deep}, 
one can first identify the transformed outcome: 
\begin{eqnarray}
    h_{\theta}( T, X) = g_{\Psi}( T, X) + \mathbb{E}[\varepsilon( U ) \mid X].
    \label{eq:counterfacutal}
\end{eqnarray}

\begin{assumption}[Encouragement-Based Instrumental Variable] \label{ass:encourage}
    The adopted encouragement policies $e \in \mathcal{E}$ serve as IVs, which only positively motivate the choice of treatments, without directly affecting the outcome response, which satisfies the following three IV conditions: \\
(a) \emph{Relevance}: IVs would directly affect $T^{e}$, i.e., $T^{e} \notindep e $; \\
(b) \emph{Exclusion}: IVs do not directly affect $Y^{e}$, i.e., $Y^{e_i}(t, x) = Y^{e_j}(t, x)$ for $e_i \not = e_j$ and all $t$ and $x$; \\
(c) \emph{Independence}: IVs are conditional independent of the error $\varepsilon(U)$, i.e., $e \indep \varepsilon(U) \mid X$.
\end{assumption}

\begin{assumption}[Additive Noise in \cite{hartford2017deep}] \label{ass:additive}
$\varepsilon( u_i^{(e_k)})$ embeds the confounding effect from $u_i^{(e_k)}$ as an additive noise term with  $\mathbb{E}[\varepsilon( U^{(e_k)})]=0$.
\end{assumption}

Under Assumptions \ref{ass:encourage} and \ref{ass:additive}, the encouragements could be seen as valid IVs, Then the CATE is identified by $\text{CATE}(T,X) = h_{\theta}( T, X) - h_{\theta}(0, X)$.

\begin{theorem} \label{theorem:3}
    When the relevance between $T$ and $e$ is strong, the unique solution ${ h }_{\theta}(T,X)$ is identified by the inverse problem given $\mathbb{E}[Y \mid e,T,X]$ and $F(T \mid e, X)$:
\begin{eqnarray}
\Scale[1.0]{\mathbb{E}[Y \mid e, T,X]} = \Scale[1.0]{\int \left[ {h}_{\theta}(T,X) \right] dF(T \mid e,X)} 
\end{eqnarray}
where, $dF(T \mid e,X)$ is the conditional treatment distribution. The proof can be found in \cite{newey2003instrumental}.
\end{theorem}

However, discrete encouragements $e$ may only introduce a minor exogenous disturbance to the continuous treatment that is too small to accurately estimate CATE. To address this issue, we propose a novel discrete encouragement algorithm by combining Theorems \ref{theorem:1} and \ref{theorem:3} to extend the moment conditions in Eq. \eqref{eq:gmm} to a non-linear setting:
\begin{eqnarray}
\label{eq:14}
\Scale[0.85]{
g_R\left(\theta, \xi\right)=\left[\begin{array}{c}
\mathbb{E}[({Y}^{(e_0)} - h_{\theta}( T^{(e_0)}, X^{(e_0)})) r_\xi({X}^{(e_0)})] \\
\cdots \\
\mathbb{E}[({Y}^{(e_K)} - h_{\theta}( T^{(e_K)}, X^{(e_K)}))r_\xi({X}^{(e_K)})]
\end{array}\right]
}. \\
\Scale[0.88]{
g_{\mathcal{E}}\left(\theta\right) = 
\left[\begin{array}{c}
\mathbb{E}[Y^{(e_i)} - \hat{Y}_{\theta}^{(e_i)}]  +
\mathbb{E}[Y^{(e_j)} - \hat{Y}_{\theta}^{(e_j)}]  \\
\mathbb{E}[Y^{(e_i)} - \hat{Y}_{\theta}^{(e_i)}]  -
\mathbb{E}[Y^{(e_j)} - \hat{Y}_{\theta}^{(e_j)}] 
\\
\operatorname{Var}[Y^{(e_i)} - \hat{Y}_{\theta}^{(e_i)}]  -
\operatorname{Var}[Y^{(e_j)} - \hat{Y}_{\theta}^{(e_j)}]
\end{array}\right]_{i \not= j}}.
\label{eq:15}
\end{eqnarray}
where $\Scale[0.92]{\hat{Y}_{\theta} = h_{\theta}( T, X)}$ and $r_\xi(\cdot)$ is the representations of $X$, providing non-linear moments. Eq. \eqref{eq:15} ensures the expectation of residual is zero and independent of $\mathcal{E}$.
\begin{eqnarray} \label{eq:ours}
\Scale[0.87]{
\left({\theta}^\ast, {\xi}^\ast \right)
= 
\arg \min_{\hat{\theta}} \sup_{\hat{\xi}}
\left[l(Y,\hat{Y}_{\theta}) 
+
g_R^\prime W_R g_R
+ 
g_{\mathcal{E}}^\prime W_{\mathcal{E}} g_{\mathcal{E}}
\right]
},
\end{eqnarray}

where $l(\cdot)$ represents the cross-entropy loss for binary outcomes or mean squared error for continuous outcomes, while the moments constraints $g_R^\prime W_R g_R$ and $g_{\mathcal{E}}^\prime W_{\mathcal{E}} g_{\mathcal{E}}$ act as penalties aiding in the estimation of $\hat{Y}_\theta$, where $\{W_R,W_{\mathcal{E}}\}$ are non-negative definite weighting matrices, with the optimal $W^\ast$ determined by the moments covariance matrix.

\begin{corollary} \label{corollary:1}
    Under Assumptions \ref{ass:encourage} \& \ref{ass:additive}, the result of the estimated $\theta^\ast$ in Eq. \eqref{eq:ours} equals ${ h  }_{\theta}(T,X)$ .
\end{corollary}

\begin{proof}
    Under Assumptions \ref{ass:encourage} and \ref{ass:additive}, Theorem \ref{theorem:3} guarantees the existence of a unique solution $h_{\theta}(T, X)$, which accounts for the correlation between the additional noise and observed covariates $\mathbb{E}[\varepsilon(U) \mid X]$. Furthermore, moment condition \eqref{eq:15} guarantees that the residual ($\epsilon = Y - \hat{Y}_\theta$) remains independent of encouragements ($e$). These conditions collectively enable us to minimize the loss function, $l(Y,\hat{Y}_{\theta})$, to approximate ${ h }_{\theta}(T,X)$ accurately.
\end{proof}

Our algorithm differs from DeepGMM \cite{2019deepgmm} and AGMM \cite{lewis2020agmm}, which require $X$ to be exogenous variables. However, in real application, this condition is hard to satisfy. Besides, these methods ignore keeping the residual expectation to be zero while minimizing the regression error. 
To this end, under nonrandom encouragements and continuous treatments, we develop novel theory and algorithms for identifying and estimating CATE. When the covariates $X^{(e)}$ shift slightly across encouragements, we reweight samples to estimate causal effects, with the Independence Assumption \ref{ass:encourage}(c).

\section{Methodology}
Combining observational and encouragement data in $\mathcal{D} = \{ \mathcal{D}^{(e_k)} \}_{e_k \in \mathcal{E}}$, we follow the theoretical insights from the previous sections to train neural networks $h_{\theta}$ with moment constraints for Encouragement-based Counterfactual Regression (EnCounteR). Specifically, our model's overall architecture comprises the following components: (1) Covariate balancing under Independence Assumption \ref{ass:encourage}(c); (2) full moment constraints with adversarial representation matrices; (3) counterfactual regression with moment constraints. Next, we will introduce each module step by step.

\subsection{Reweighting for Covariate Balance}
\label{sec:41}
As depicted in Figure \ref{fig:abs}, we collect large observational data $\mathcal{D}^{(e_0)}$ from previous samples and implement $K$ encouragements $\{\mathcal{D}^{(e_{k})}\}_{1 \leq k \leq K}$ in new samples to examine the causal effect of $T$ on $Y$. However, samples under different encouragements may exhibit slight covariate shifts, such as minor differences between two different classes in the same school. Therefore, we introduce the following Reweighting module to balance observed covariates across various environments: 
\begin{eqnarray}
\Scale[0.85]{
\mathcal{L}_{\omega} = \sum_{j \not = k}
(\mathbb{E}_{\omega}X^{(e_j)} - \mathbb{E}_{\omega}X^{(e_k)})^2 
+ 
(\text{Cov}_{\omega}X^{(e_j)} - \text{Cov}_{\omega}X^{(e_k)})^2}, \nonumber 
\end{eqnarray}
\vspace{-19pt}
\begin{eqnarray}
\Scale[0.9]{
\mathbb{E}_{\omega}X^{(e_k)} = \omega^{(e_k)^\prime} X^{(e_k)}, \text{Cov}_{\omega}X^{(e_k)} = {\tilde{X}^{(e_k)^\prime}} \omega \tilde{X}^{(e_k)}}, \label{eq:17}\\
\Scale[0.9]{
\omega^{(e_k)} = [(1+3\sigma(w^{(e_k)}))/2] / [\sum_i^{n_j}(1+3\sigma(w^{(e_k)}))/2]}, \nonumber
\end{eqnarray}
where $\sigma(\cdot)$ is the sigmoid function, and $w^{(e_k)}$ are trainable parameters with $n_k$ units. The term $\frac{1+3\sigma(w)}{2} \in [\frac{1}{2}, 2]$ serves to limit extreme values during the reweighting process.

\begin{algorithm*}[t]
    \caption{EnCounteR: Encouragement-Based Counterfactual Regression}
    \label{algorithm}
    \begin{algorithmic}
	\STATE \textbf{Input:} Encouragement designs $\mathcal{D} = \{ \mathcal{D}^{(e_k)} \}_{e_k \in \{e_0,e_1,\cdots,e_K\}}$, each with $\mathcal{D}^{(e_k)} = \{x_i^{(e_k)}, t_i^{(e_k)}, y_i^{(e_k)} \}_{i=1}^{n_k}$; Hyper-parameters $\{\text{d}_\text{h},\text{d}_\text{r},\alpha\}$;Trainable Weighting Vectors $\omega^{(e_k)} = \frac{1+3\sigma(w^{(e_k)})} {\sum_i^{n_k}(1+3\sigma(w^{(e_k)}))}$ with default $w^{(e_k)}=1$; Counterfactual Regression Network $h_\theta(\cdot)$ with Trainable Parameters $\theta$; Adversarial Representation Network $r_\xi(\cdot)$ with Trainable Parameters $\xi$; Reweighting-Training-Epoch $\mathcal{I}_1=\text{10}$; Adversarial-Training-Epoch $\mathcal{I}_2=\text{100}$; Full-Training-Epoch $\mathcal{I}_3=\text{1,000}$. 
        \STATE \textbf{Output:} Counterfactual Outcome Function $\hat{Y}_\theta(t, X)=h_{\theta}(do(t), X)$, and Conditional Average Treatment Effect $\text{CATE}(t, X)=\hat{Y}_\theta(t, X)-\hat{Y}_\theta(0, X)$. 
	\STATE \textbf{Loss function:} $\mathcal{L}_\omega$ in Eq. \eqref{eq:17}, $\mathcal{L}_R$ in Eq. \eqref{eq:20},  and $\mathcal{L}=\mathcal{L}_{\text{REG}} + \alpha (\mathcal{L}_{\mathcal{E}} + \mathcal{L}_{X} + \mathcal{L}_{R})$ in Eq. \eqref{eq:objective}.
	\STATE \textbf{Reweighting for Covariate Balance:}
	\FOR{$\text{itr}=1$ {\bfseries to} $\mathcal{I}_1$}
        \STATE $\mathbb{E}_{\omega}[X^{(e_k)}] = \sum_i^{n_k}\omega_i^{(e_k)} x_i^{(e_k)}$, 
        \STATE $\text{Cov}_{\omega}[X^{(e_k)}] = \left[\sum_i^{n_k}\omega_i^{(e_k)}(x_{i,a}^{(e_k)}-\mathbb{E}_{\omega}[X^{(e_k)}_a])(x_{i,b}^{(e_k)}-\mathbb{E}_{\omega}[X^{(e_k)}_b])\right]_{1 \leq a,b \leq \text{d}_\text{x}}$, where $a$ denotes $a$-th variable in $X$, 
	\STATE $\mathcal{L}_{\omega} = \sum_{j \not = k}
        (\mathbb{E}_{\omega}[X^{(e_j)}] - \mathbb{E}_{\omega}[X^{(e_k)}])^2 
        + 
        (\text{Cov}_{\omega}[X^{(e_j)}] - \text{Cov}_{\omega}[X^{(e_k)}])^2$,
	\STATE update $\omega \leftarrow {\rm{Adam}}\{\mathcal{L}_{\omega}\}$ using Adaptive Moment Estimation.
	\ENDFOR
	\STATE \textbf{Counterfactual Regression:}
	\FOR{$\text{itr}=1$ {\bfseries to} $\mathcal{I}_3$}
        \IF{$\text{itr} \leq \mathcal{I}_2$}
        \STATE $ \{X^{(e_k)}, T^{(e_k)}, Y^{(e_k)} \}_{0 \leq k \leq K} \rightarrow 
        \mathcal{L}_{R} =  g_R^\prime\left(\theta, \xi\right) \cdot W_R \cdot g_R\left(\theta, \xi\right)$,
        \STATE update $\xi \leftarrow {\rm{Adam}}\{-\mathcal{L}_R\}$ in representation network $h_\theta(\cdot)$ using Adaptive Moment Estimation.
        \ENDIF
	\STATE $ \{X^{(e_k)}, T^{(e_k)}, Y^{(e_k)} \}_{0 \leq k \leq K} \rightarrow 
        \mathcal{L}=\mathcal{L}_{\text{REG}} + \alpha (\mathcal{L}_{\mathcal{E}} + \mathcal{L}_{X} + \mathcal{L}_{R})$,
        \STATE update $\theta \leftarrow {\rm{Adam}}\{\mathcal{L}\}$ in counterfactual regression $h_\theta(\cdot)$ using Adaptive Moment Estimation
	\ENDFOR
\end{algorithmic}
\end{algorithm*}

\subsection{Moment Constraint Learning}
Following the weight $\omega$ from Eq. \eqref{eq:17}, we define $\mathbb{E}{\omega}$ as weighted expectation and $\text{Var}{\omega}$ as weighted variance, constructing moment constraints to learn $h_\theta(T, X)$ and $r_\xi(X)$.

(I) Encouragement-Independent Moments:
\begin{eqnarray} \label{eq:18}
\mathcal{L}_{\mathcal{E}} =  g_{\mathcal{E}}^{\prime}\left(\theta\right) \cdot W_{\mathcal{E}} \cdot g_{\mathcal{E}}\left(\theta\right), 
\end{eqnarray}
\vspace{-16pt}
\begin{eqnarray}
\Scale[0.88]{
g_{\mathcal{E}}\left(\theta\right) = 
\left[\begin{array}{c}
\mathbb{E}_{\omega}[Y^{(e_i)} - \hat{Y}_{\theta}^{(e_i)}]  +
\mathbb{E}_{\omega}[Y^{(e_j)} - \hat{Y}_{\theta}^{(e_j)}]  \\
\mathbb{E}_{\omega}[Y^{(e_i)} - \hat{Y}_{\theta}^{(e_i)}]  -
\mathbb{E}_{\omega}[Y^{(e_j)} - \hat{Y}_{\theta}^{(e_j)}] 
\\
\operatorname{Var}_{\omega}[Y^{(e_i)} - \hat{Y}_{\theta}^{(e_i)}]  -
\operatorname{Var}_{\omega}[Y^{(e_j)} - \hat{Y}_{\theta}^{(e_j)}]
\end{array}\right]_{i \not= j}}, \nonumber
\end{eqnarray}
where $\hat{Y}_{\theta} = h_{\theta}(T,X)$. Eq. \eqref{eq:18} guarantees that the residual ($\epsilon = Y - \hat{Y}_\theta$) remains independent of encouragements ($e$).

(II) Covariate-Independent Moments:
\begin{eqnarray} \label{eq:19}
\mathcal{L}_{X} =  g_X^\prime\left(\theta, \xi\right) \cdot W_X \cdot g_X\left(\theta, \xi\right),
\end{eqnarray}
\vspace{-16pt}
\begin{eqnarray}
\Scale[0.85]{
g_X\left(\theta, \xi\right)=\left[\begin{array}{c}
\mathbb{E}_{\omega}[({Y}^{(e_0)} - h_{\theta}( T^{(e_0)}, X^{(e_0)})) \tilde{X}^{(e_0)}] \\
\cdots \\
\mathbb{E}_{\omega}[({Y}^{(e_K)} - h_{\theta}( T^{(e_K)}, X^{(e_K)})) \tilde{X}^{(e_K)}]
\end{array}\right]
}.
\nonumber
\end{eqnarray}
Equation \eqref{eq:19} ensures that the residual ($\epsilon = Y - \hat{Y}_\theta$) and covariates are linearly independent.

(III) Adversarial Representation-Independent Moments:
\begin{eqnarray} \label{eq:20}
\mathcal{L}_{R} =  g_R^\prime\left(\theta, \xi\right) \cdot W_R \cdot g_R\left(\theta, \xi\right),
\end{eqnarray}
\vspace{-16pt}
\begin{eqnarray}
\Scale[0.85]{
g_R\left(\theta, \xi\right)=\left[\begin{array}{c}
\mathbb{E}_{\omega}[({Y}^{(e_0)} - h_{\theta}( T^{(e_0)}, X^{(e_0)})) r_\xi({X}^{(e_0)})] \\
\cdots \\
\mathbb{E}_{\omega}[({Y}^{(e_K)} - h_{\theta}( T^{(e_K)}, X^{(e_K)}))r_\xi({X}^{(e_K)})]
\end{array}\right]
}.
\nonumber
\end{eqnarray}
In complex non-linear settings, the underlying independence assumptions typically entail an infinite set of moment conditions. Consequently, we employ Adversarial Representation Learning to learn non-linear factors $R = r_{\xi}(X) \in \mathbb{R}^{\text{d}_\text{r}}$ for adaptively constructing the top-$\text{d}_\text{r}$ moment conditions in minimax criterion (see Eq. \eqref{eq:objective}). In Eqs. (\ref{eq:18}-\ref{eq:20}), $W_{\mathcal{E}}$, $W_X$ and $W_R$ are non-negative definite weighting matrices, and the optimal $W^\ast$ depends on the moments covariance matrix.

\subsection{Counterfactual Regression} \label{sec:regression}
Before proceeding with counterfactual regression $h_\theta(\cdot)$, we conduct a statistical test to check if the mean and covariance of $X$ are independent of encouragements; if not, we perform a preprocessing step to learn $\omega$ for achieving covariate balance, as detailed in Section \ref{sec:41}.
We employ two-layer neural networks with ELU activation, where each layer comprises $\text{d}_\text{h}$ hidden units for $\text{d}_\text{r}$-dimensional Representation $R=r_\xi(X)$ and Counterfactual Regression $\hat{Y}_\theta=h_\theta(T,X)$:
\begin{eqnarray}
\mathcal{L}_{\text{REG}} =  \mathbb{E}_{\omega}[l(Y,h_{\theta}(T, Y))].
\end{eqnarray}
Following Theorems \ref{theorem:3} and Corollary \ref{corollary:1}, the complete objective function is formulated as follows: 
\begin{eqnarray} \label{eq:objective}
\Scale[1.0]{
\arg\min_{{\theta}} \sup_{{\xi}}
\mathcal{L} = \mathcal{L}_{\text{REG}} 
+ \alpha (\mathcal{L}_{\mathcal{E}}
+ \mathcal{L}_{X} 
+ \mathcal{L}_{R}) 
},
\end{eqnarray}
where $\alpha$ is a hyper-parameter to control the strength of moments constraints.

\subsection{Implementation}

In this paper, we use two-layer neural networks with ELU activation, with each layer containing $\text{d}_\text{h}$ hidden units, for both Counterfactual Regression $\hat{Y}_\theta=h_\theta(T,X)$ and $\text{d}_\text{r}$-dimensional Representation $R=r_\xi(X)$.  We adopt full-batch training for the proposed EnCounteR algorithm, optimize it with the objective function \eqref{eq:objective}, and set the maximum number of training epochs to 1,000. EnCounteR contains three hyperparameters, i.e., $\text{d}_\text{r} \in \{1,5,8,10,12,20\}$, $\text{d}_\text{h} \in \{16,32,64,128,256\}$, and $\alpha \in \{1,2,5,12,15,20\}$. We utilize the minimum regression error on the validation dataset to optimize hyper-parameters. The pseudocode is placed in Algorithm \ref{algorithm}.

\textbf{Hardware used}: Ubuntu 16.04.3 LTS operating system with 2 * Intel Xeon E5-2660 v3 @ 2.60GHz CPU (40 CPU cores, 10 cores per physical CPU, 2 threads per core), 256 GB of RAM, and 4 * GeForce GTX TITAN X GPU with 12GB of VRAM. \textbf{Software used}: Python 3.7.15 with TensorFlow 1.15.0, Pytorch 1.7.1, and NumPy 1.18.0.

\begin{figure*}[t]
\begin{center}
\includegraphics[width=0.95\linewidth]{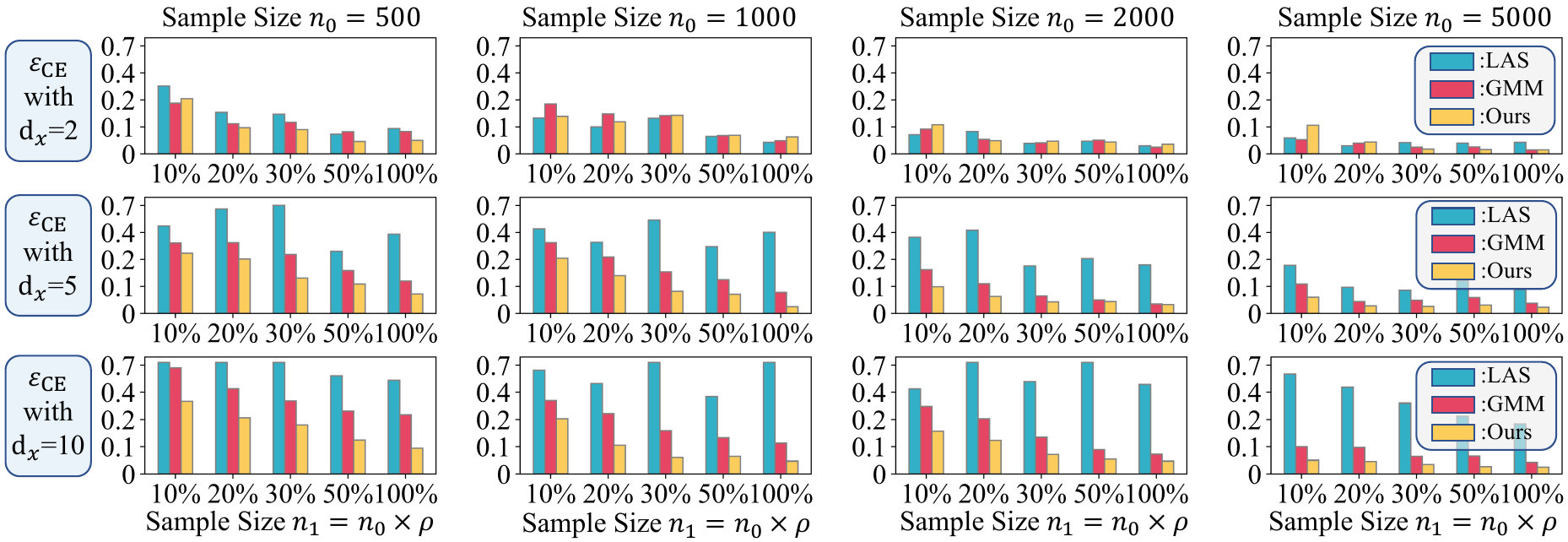}
\end{center}
\caption{Results ($\varepsilon_{\text{CE}}$) of LAS, GMM, and Our EnCounteR in Linear Simulations, with varying sample sizes $n_0 \in \{500,1000,2000,5000\}$ for observational dataset $\mathcal{D}^{(e_0)}$ and varying sample sizes $n_1 = n_0 \times \rho$ with $\rho = \{10\%,20\%,30\%,50\%,100\%\}$ for encouragement experiments $\mathcal{D}^{(e_1)}$ across various dimensions $\text{d}_\text{x} = \{2,5,10\}$ of $X$.
}
\label{fig:linear}
\end{figure*}

\section{Empiracal Experiments}
\subsection{Baselines and Metrics}

We compare the proposed \textbf{EnCounteR} algorithm with two groups of methods. One group is \emph{instrument-based methods}: (1) 
\textbf{KernelIV} \cite{singh2019kernel} and \textbf{DualIV} \cite{muandet2020dual}  implement kernel ridge regression derived from reproducing kernel Hibert spaces; (2) \textbf{DeepGMM} \cite{2019deepgmm} and \textbf{AGMM} \cite{lewis2020agmm} construct structural functions and search moment conditions via adversarial training; (3) \textbf{DeepIV} \cite{hartford2017deep}, \textbf{DFIV} \cite{xu2020dfiv} and \textbf{CBIV} \cite{wu2022instrumental} run two-stage regression using deep neural networks. The other group is \emph{covariate-based methods}: (1) \textbf{CFRNet} \cite{shalit2017estimating} and \textbf{DRCFR} \cite{hassanpour2020learning} use mutual information to learn balanced representations in continuous cases, while \textbf{VCNet} \cite{nie2021vcnet} is tailored to continuous treatment, preserving the continuity of the estimated counterfactual curve; (2) \textbf{CEVAE} \cite{louizos2017causal} and \textbf{TEDVAE} \cite{zhang2021treatment} employ latent variable modeling to concurrently estimate the latent space summarizing confounders and the causal effects; (3) \textbf{KerIRM} and \textbf{KerHRM} aim to identify causally invariant relationships in different environments, with the former using known encouragement labels and the latter not using them \cite{arjovsky2019invariant,liu2021heterogeneous,liu2021kernelized}. 
Additionally, we employ the \textbf{VANILLA} network regression as a baseline. 

In this section, we use three key metrics for evaluation purposes: $\varepsilon_{\text{CE}} = |\hat{\psi}_t-\psi_t |$ measures causal parameter estimation accuracy in linear simulations; $\varepsilon_{\text{CFR}} = \mathbb{E}(\hat{Y}_\theta(do(t),X)-Y(do(t),X) )^2$ assesses the precision of counterfactual outcome predictions using mean square error, where $do(t)$ denotes do operations randomly sampled from a uniform distribution $\text{Unif}[0, 1]$; and the Precision in Estimation of Heterogeneous Effect is measured by $\varepsilon_{\text{PEHE}} = \sqrt{
\mathbb{E}(\text{CATE}(do(t),X) - \text{CATE}(do(0),X))^2
}$.

\subsection{Experiments on Linear Simulations}
\label{sec:lineara}
\paragraph{Datasets} 
In linearity scenarios, we collect observational samples $\mathcal{D}^{(e_0)}$ with varying sizes $n_0 \in \{500,1000,2000,5000\}$. Additionally, we conduct a single encouragement experiment $e_1$ on a smaller dataset $\mathcal{D}^{e_1}$, where we manipulate the experimental data size by setting $n_1=n_0 \times \rho$ with different proportions $\rho = \{10\%,20\%,30\%,50\%,100\%\}$ to investigate the impact of sample size on performance of our EnCounteR. Subsequently, we generate a combined dataset $\{\mathcal{D}^{(e_k)}\}_{k \in \{0,1\}} = \{X^{(e_k)}, U^{(e_k)}, T^{(e_k)}, T^{(e_k)}\}_{k \in \{0,1\}}$ with $T^{(e_k)} = \phi_{x}^{(e_k)\prime} X^{(e_k)} +\phi_{u}^{(e_k)\prime} U^{(e_k)}$ and $Y^{(e_k)} = \psi_t T^{(e_k)} + \psi_{x}^\prime X^{(e_k)} +\psi_{u}^\prime U^{(e_k)}$,
where $\text{d}_\text{x}$-dimensional random vector $X$ and $\text{d}_u$-dimensional $U$ are generated from a Normal Distribution $\mathcal{N}(0,1)$ with common covariance of 0.3, and we set $\text{d}_\text{x} \in \{2,5,10\}$ and $\text{d}_\text{u} = 2$. The corresponding coefficients $\{\phi_{x}^{(e_k)}, \phi_{u}^{(e_k)}, \psi_{x}, \psi_{u}\}$ are independently sampled from a Uniform distribution $\text{Unif}(0,1)$. In our experiments, $\psi_t$ represents the causal parameter of interest, and we set it to $\psi_t=0.5$. We leverage $\{\mathcal{D}^{(e_k)}\}_{k \in \{0,1\}}$ as training data, reserving 10\%-30\% of $\mathcal{D}^{(e_0)}$ as validation data, and generate 20,000 additional samples with random treatments $do(t) \sim \text{Unif}[0, 1]$ and its corresponding outcome $Y(do(t),X)$ as testing data. 
We conduct 10 independent replications.

\begin{table*}[t]
{
\caption{Results ($\text{mean}_{\pm \text{std}}$) of $\epsilon_{\text{CFE}}$ and $\epsilon_{\text{PEHE}}$ on Simulation, IHDP and ACIC Datasets.}
\label{tab:main}
\begin{center}
\setlength{\tabcolsep}{2mm}
{\small
\begin{tabular}{c|cc|cc|cc}
\toprule
 & \multicolumn{2}{c|}{\bf Simulation (\textsc{Mult})} & \multicolumn{2}{c|}{\bf IHDP}
& \multicolumn{2}{c}{\bf ACIC} \\
\midrule
Methods & \bf
$\epsilon_{\text{CFE}}$ & \bf $\epsilon_{\text{PEHE}}$ & \bf $\epsilon_{\text{CFE}}$ & \bf $\epsilon_{\text{PEHE}}$ & \bf $\epsilon_{\text{CFE}}$ & \bf $\epsilon_{\text{PEHE}}$ \\
\midrule

\bf KernelIV & $ 17.44_{\pm 2.147} $ & $ 0.611_{\pm 0.153} $  & $ 3.808_{\pm 1.279} $ & $ 0.581_{\pm 0.046} $  & $ 38.82_{\pm 2.457} $ & $ 0.602_{\pm 0.023} $ \\
\bf DualIV & $ 92.64_{\pm 44.39} $ & $ 2.454_{\pm 0.679} $  & $ 19.60_{\pm 4.877} $ & $ 2.537_{\pm 0.248} $  & $ 28.41_{\pm 3.384} $ & $ 0.752_{\pm 0.047} $ \\
\bf DeepGMM & $ 6.340_{\pm 2.177} $ & $ 0.584_{\pm 0.105} $  & $ 1.967_{\pm 0.514} $ & $ 0.478_{\pm 0.029} $  & $ 10.09_{\pm 1.798} $ & $ 0.551_{\pm 0.070} $ \\
\bf AGMM & \underline{$ 5.941_{\pm 0.994} $} & $ \underline{0.274_{\pm 0.045}} $  & $ 1.556_{\pm 0.252} $ & $ 0.414_{\pm 0.033} $  & $ 13.84_{\pm 1.340} $ & $ 0.375_{\pm 0.018} $ \\
\bf DeepIV & $ 19.13_{\pm 2.327} $ & $ 0.662_{\pm 0.021} $  & $ 2.065_{\pm 0.305} $ & $ 0.642_{\pm 0.024} $  & $ 40.79_{\pm 13.15} $ & $ 0.605_{\pm 0.031} $ \\
\bf DFIV & $ 11.73_{\pm 0.894} $ & $ 0.563_{\pm 0.025} $  & $ 2.928_{\pm 0.500} $ & $ 0.476_{\pm 0.025} $  & $ 24.78_{\pm 3.108} $ & $ 1.247_{\pm 0.097} $ \\
\bf CBIV & $ 11.61_{\pm 2.675} $ & $ 0.551_{\pm 0.116} $  & $ 6.540_{\pm 1.465} $ & $ 0.760_{\pm 0.195} $  & $ 11.37_{\pm 3.168} $ & $ 0.414_{\pm 0.058} $ \\
\midrule

\bf CFRNet & $ 6.600_{\pm 0.606} $ & $ 0.290_{\pm 0.041} $  & $ 3.155_{\pm 2.893} $ & $ 0.482_{\pm 0.166} $  & $ 9.305_{\pm 1.370} $ & $ 0.387_{\pm 0.079} $ \\
\bf DRCFR & $ 6.410_{\pm 0.533} $ & $ 0.310_{\pm 0.027} $  & $ 0.866_{\pm 0.298} $ & $ 0.447_{\pm 0.034} $  & $ 9.329_{\pm 1.685} $ & $ 0.348_{\pm 0.025} $ \\
\bf VCNet & $ 7.490_{\pm 0.289} $ & $ 0.309_{\pm 0.026} $  & \underline{$ 0.611_{\pm 0.128} $} & \underline{$ 0.229_{\pm 0.031} $} & \underline{$ 8.298_{\pm 2.338} $} & \underline{$ 0.263_{\pm 0.072} $} \\
\bf CEVAE & $ 9.899_{\pm 0.592} $ & $ 0.525_{\pm 0.053} $  & $ 4.585_{\pm 0.539} $ & $ 0.722_{\pm 0.038} $  & $ 21.88_{\pm 2.148} $ & $ 0.867_{\pm 0.061} $ \\
\bf TEDVAE & $ 16.24_{\pm 0.379} $ & $ 0.702_{\pm 0.013} $  & $ 6.546_{\pm 0.768} $ & $ 0.691_{\pm 0.023} $  & $ 29.77_{\pm 2.955} $ & $ 0.764_{\pm 0.019} $ \\
\bf KerIRM & $ 13.12_{\pm 2.589} $ & $ 0.479_{\pm 0.073} $  & $ 3.696_{\pm 0.978} $ & $ 0.649_{\pm 0.042} $  & $ 23.93_{\pm 2.944} $ & $ 0.627_{\pm 0.031} $ \\
\bf KerHRM & $ 17.94_{\pm 3.808} $ & $ 0.547_{\pm 0.083} $  & $ 5.383_{\pm 1.588} $ & $ 0.581_{\pm 0.077} $  & $ 24.61_{\pm 2.974} $ & $ 0.659_{\pm 0.093} $ \\

\midrule 

\bf VANILLA & 
$ 7.512_{\pm 1.048} $ & $ 0.348_{\pm 0.067} $ 
& $ 2.068_{\pm 1.917} $ & $ 0.510_{\pm 0.323} $ 
& $ 19.66_{\pm 14.98} $ & $ 0.656_{\pm 0.317} $ 
\\

\bf EnCounteR  & $ \mathbf{ 4.816_{\pm 0.609} } $ & $ \mathbf{ 0.210_{\pm 0.026} } $ 
& $ \mathbf{ 0.582_{\pm 0.130} } $ & $ \mathbf{ 0.188_{\pm 0.021} } $ 
& $ \mathbf{ 5.751_{\pm 0.606} } $ & $ \mathbf{ 0.186_{\pm 0.038} } $ 
\\

\bottomrule
\end{tabular}
}
\end{center}
}
\end{table*}

\label{sec:ex-linear}
\paragraph{Results} In the linear simulation experiments (Figure \ref{fig:linear}), we employ three parametric estimators: the linear analytical solution (\textbf{LAS}) from Theorem \ref{theorem:1}, the \textbf{GMM} reformulation in Eq. \eqref{eq:gmm}, and our \textbf{EnCounteR} in Eq. \eqref{eq:ours}. 
The LAS method relies on a substantial variation, $\beta_{T|X}^{(e_1)} - \beta_{T|X}^{(e_0)}$, and is limited to estimating $\phi_t$ using only a single $X$ variable. As data dimensions increase in Figure \ref{fig:linear}, the influence of variations in single $X$ on $T$ diminishes, which would introduce larger errors in $\varepsilon_{\text{CE}}$. 
Moreover, inaccuracies in estimating $\beta_{Y|X}^{(e_k)}$ and $\beta_{T|X}^{(e_k)}$ could exacerbate LAS errors by magnifying them further. To address the over-identification issue, we use GMM and EnCounteR reformulations to identify the causal parameter leveraging moments on residuals from full variables $X$. As shown in Figure \ref{fig:linear}, regardless of varying dimensions of $X$, both GMM and EnCounteR consistently exhibit robustness in estimating causal effects. Following Corollary \ref{corollary:1}, EnCounteR with novel moments (Eq. \eqref{eq:14}) yields more accurate estimates of causal parameters with lower variance. Furthermore, with varying encouragement proportions, $\rho=\{10\%,20\%,30\%,50\%,100\%\}$, our EnCounteR consistently performs well when $n_1 \geq n_0 \times 30\%$, greatly reducing the costs of the encouragement experiments and the computational expenses. Therefore, in subsequent studies, we set $n_k = n_0 \times 30\%$ for $k \geq 1$. 

\begin{figure*}[t]
\begin{center}
\includegraphics[width=0.88\linewidth]{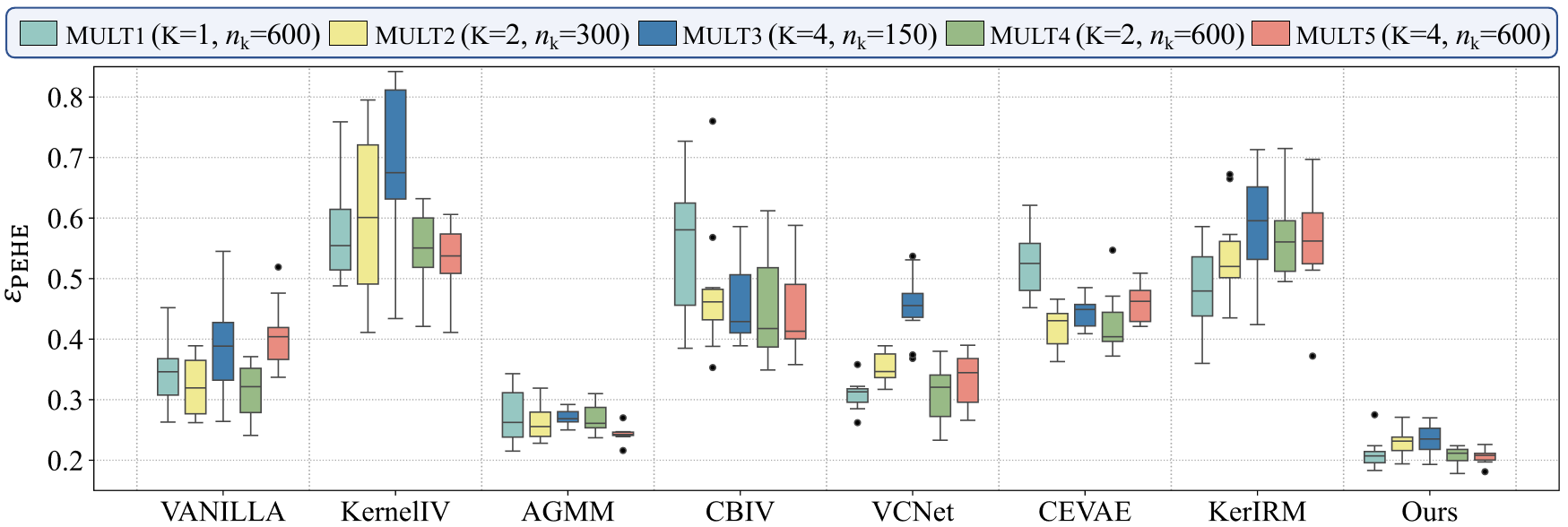}
\end{center}
\caption{Box Plot of $\epsilon_{\text{PEHE}}$ in Simulation (\textsc{Mult}s): Explore Varying Encouragements $K$ and Varying Data Volume $K \times n_k$.}
\label{fig:box}
\end{figure*}

\subsection{Experiments on Complex Datasets} \label{sec:53}

In complex non-linear setting with heterogeneous treatment effects, we evaluate the EnCounteR method on Simulations (\textsc{Mult}s, \textsc{Poly}s, \textsc{Abs}s, \textsc{Sin}s) and two widely-adopted benchmarks: \textbf{\textsc{IHDP}} \cite{hill2011bayesian,shalit2017estimating} containing 747 samples with 25 variables (selected as 5 observed confounders and 20 unmeasured confounders), \textbf{\textsc{ACIC} 2016} \cite{dorie2019automated} containing 4,802 samples with 58 variables (selected as 12 observed confounders and 46 unmeasured confounders).

\paragraph{Simulation Datasets} First, we introduce the generation process of \textbf{Simulations (\textsc{Mult}s) with Covariate Shifts} across different encouragements $ e_k \in \{e_0, e_1, \cdots, e_K\}$. For each encouragements $e_k$, we generate the observed covariates by $X^{(e_k)} \sim  \mathcal{N}(\mu_x^{(e_k)},\Sigma_x^{(e_k)}), \mu_x^{(e_k)} \sim \text{Unif}(-0.2,0.2)$ with:
\begin{eqnarray}
\nonumber
\Sigma_x \quad = \quad
\begin{pmatrix}
\sigma_x^{(e_k)} & 0.3 & \cdots & 0.3 \\
0.3 & \sigma_x^{(e_k)} & \cdots & 0.3 \\
\vdots & \vdots & \ddots & \vdots \\
0.3 & 0.3 & \cdots & \sigma_x^{(e_k)}
\end{pmatrix},
\end{eqnarray}
where $\sigma_x^{(e_k)} \sim \text{Unif}(0.7,1.3)$. Then, the unmeasured confounders would be $ U_i^{(e_k)} \sim \mathcal{N}(0.3(X^{(e_k)}_{2i-1}+X^{(e_k)}_{2i}),1) $, where subscript $i$ denotes the $i$-th variable in $U$ and $i \in \{1,2,\cdots, \text{d}_\text{u}\}$. 
In the main experiments, guided by the findings in Section \ref{sec:lineara}, we set observational data at $n_0 = \text{2,000}$ and experimental data at $n_k = 600$ for $1 \leq k \leq K$, with parameters $K = 1$, $\text{d}_\text{x} = 5$, and $\text{d}_\text{u} = 3$.
Define $C \in \mathbb{R}^{\text{d}_\text{x}+\text{d}_\text{u}}$ as the concatenation of all confounders $X$ and $U$, we generate the treatments and outcomes with multiplicative cross-terms:  
{ \small
\begin{eqnarray}
\nonumber
T^{(e_k)} & = & \left| \sum_{i=1}^{\text{d}_\text{x}+\text{d}_\text{u}-1} \left[\phi_i^{(e_k)}C_i^{(e_k)}  C_{i+1}^{(e_k)}\right] + \sum_{i=1}^{\text{d}_\text{x}+\text{d}_\text{u}} \left[\phi_{i}^{(e_k)}C_{i}^{(e_k)}\right] \right|, \\
Y^{(e_k)}_{\textsc{Mult}} & = & T^{(e_k)} \times (0.5+X_0^{(e_k)}) + \sum_{i=1}^{\text{d}_\text{x}-1} \left[\psi_iX_i^{(e_k)}  X_{i+1}^{(e_k)}\right] 
\nonumber \\ 
&+& \sum_{i=1}^{\text{d}_\text{u}-1} \left[\psi_iU_i^{(e_k)}  U_{i+1}^{(e_k)}\right] + \sum_{i=1}^{\text{d}_\text{x}+\text{d}_\text{u}} \left[\psi_iC_i^{(e_k)}\right],
\nonumber
\end{eqnarray}
}where, the coefficients $\{\phi_{1 \cdots (\text{d}_\text{x}+\text{d}_\text{u})}^{(e_k)}, \psi_{1 \cdots (\text{d}_\text{x}+\text{d}_\text{u})}\}$ are independently sampled from a Uniform distribution $\text{Unif}(0,1)$
In the above equations, we name this simulation with non-linear multiplicative cross-terms as \textbf{\textsc{Mult}}. 

\textbf{\textsc{Mult}s with Varying Encouragements $K$ and Data Volume $K \times n_k$ for $1 \leq k \leq K$}. Moreover, we generate four \textsc{Mult} datasets with more encouragements ($K>1$) and different sample sizes $n_k$, keeping $n_0=\text{2,000}$. We name the above original dataset as \textsc{Mult1} with $K=1$, $n_k=600$ and total volume $K \times n_k = 600$ for $k>0$. To further explore the effects of increased encouragements and varying data volumes, we generate four additional datasets with more encouragements ($K>1$), keeping the observational data size at $n_0=\text{2,000}$ and varying the sample sizes $n_k$ for $1 \leq k \leq K$. Keeping a fixed total volume of encouragement data at $K \times n_k = 600$ for $k>0$, we construct two additional datasets: \textsc{Mult2} with $K=2$ and $n_k=300$ for $k>0$, and \textsc{Mult3} with $K=4$ and $n_k=150$ for $k>0$. This allows us to analyze the effects of varying the number of encouragements while keeping the total volume of encouragement data constant.
With fixed size $n_k=600$, we create two additional datasets: \textsc{Mult4} with $K=2$ and a total encouragement data volume of $K \times n_k=\text{1,200}$ for $k>0$, and \textsc{Mult5} with $K=4$ and a total encouragement data volume of $K \times n_k=\text{2,400}$ for $k>0$.
This enables us to conduct a comprehensive analysis of the influence of varying numbers of encouragements and total encouragement data volumes on our study's outcomes.

\textbf{Simulations with Additional Non-linear Terms: \textsc{Poly}, \textsc{Abs} and \textsc{Sin}}. 
To simulate real-world conditions, we add three additional non-linear terms in the outcome function for assessing the EnCounteR algorithm:
{\small
\begin{eqnarray}
Y^{(e_k)}_{\textsc{Poly}} &=& Y^{(e_k)}_{\textsc{Mult}} + T^{(e_k)} \times \left(X_1^{(e_k)}\right)^2 + \sum_{i=1}^{\text{d}_\text{x}+\text{d}_\text{u}} \left[\psi_i\left(C_i^{(e_k)}\right)^2\right]  / \text{d}_\text{x}, 
\nonumber \\
Y^{(e_k)}_{\textsc{Abs}} &=& Y^{(e_k)}_{\textsc{Mult}} + T^{(e_k)} \times \left|X_1^{(e_k)}\right| + \sum_{i=1}^{\text{d}_\text{x}+\text{d}_\text{u}} \left[\psi_i\left|C_i^{(e_k)}\right|\right] / \text{d}_\text{x},
\nonumber
\\
Y^{(e_k)}_{\textsc{Sin}} &=& Y^{(e_k)}_{\textsc{Mult}} + T^{(e_k)} \times \text{sin}\left(X_1^{(e_k)}\right) + \sum_{i=1}^{\text{d}_\text{x}+\text{d}_\text{u}} \left[\psi_i\text{sin}\left(C_i^{(e_k)}\right)\right] / \text{d}_\text{x}. 
\nonumber
\end{eqnarray}
}
We name these three simulations with additional non-linear terms as \textsc{Poly}, \textsc{Abs} and \textsc{Sin}, respectively. For each data, we combine $\{\mathcal{D}^{e_k}\}_{e_k \in \mathcal{E}}$ as training data, reserving 10\%-30\% of $\mathcal{D}^{e_0}$ as validation data, and generate 20,000 additional samples with random treatments $do(t) \sim \text{Unif}[0, 1]$ and the outcomes $Y(do(t),X)$ as testing data.

\paragraph{Real-World Data}
Although massive open online courses (MOOCs) like Coursera, edX, and Udacity bring a deluge of data about student behavior in classrooms \cite{breslow2013studying,kizilcec2014encouraging,reich2015rebooting}, due to concerns over information privacy, we lack access to complete data on student behavior in MOOCs. Furthermore, based on the publicly available data, specifically \url{https://doi.org/10.7910/DVN/26147} and \url{http://moocdata.cn/data/user-activity}, we cannot construct complete encouragement data for evaluating our algorithm. Therefore, similar to previous work \cite{shalit2017estimating,yao2021survey}, we validate our algorithm on the IHDP and ACIC2016 datasets.

\textbf{IHDP}. 
The Infant Health and Development Program (IHDP\footnote{IHDP datasets are available at: \url{https://www.fredjo.com/}.}) \cite{hill2011bayesian} studies the effect of specialist home visits on the future cognitive test scores of premature infants, which comprises 747 units, with 139 in the treated group and 608 in the control group.
There are 25 pre-treatment variables ($C \in \mathbb{R}^{25}$) related to the children and their mothers. In the IHDP study, to create multi-encouragement data, the large control group is used as $\mathcal{D}^{(e_0)}$, and the small treated group as $\mathcal{D}^{(e_1)}$. We select $\text{d}_\text{x} =5$ continuous variables from the IHDP as observed covariates and use the expected potential outcomes $m_0$ for control outcomes and $m_1$ for treated outcomes as unmeasured confounding effects from the remaining $\text{d}_\text{u} = 20$ unmeasured variables. The encouraged treatments are from $T^{(e_k)} = | \sum_{i=1}^{\text{d}_\text{x}-1} [\phi_i^{(e_k)}X_i^{(e_k)} X_{i+1}^{(e_k)}] + \sum_{i=1}^{\text{d}_\text{x}} [\phi_{i}^{(e_k)}X_{i}^{(e_k)}] + m_0^{(e_k)} |$, and outcomes are determined by $Y^{(e_k)}_{\textsc{IHDP}} = T^{(e_k)} \times (0.5+X_0^{(e_k)}) + \sum_{i=1}^{\text{d}_\text{x}-1} [\psi_i X_i^{(e_k)} X_{i+1}^{(e_k)}] + \sum_{i=1}^{\text{d}_\text{x}} [\psi_i X_{i}^{(e_k)}] + m_1^{(e_k)}$. From the control group $\mathcal{D}^{(e_0)}$, We split 75 samples for validation data and another 75 for pre-testing data, leaving $n_0 = \text{458}$ samples as encouragement data with $e_0$, while maintaining $n_1 = \text{139}$ in the treated group. The pre-testing data is replicated 100 times to create 7,500 samples with random treatments $do(t) \sim \text{Unif}[0, 1]$, and for these samples, we regenerate the corresponding outcomes $Y(do(t),X)$ to be used as testing.

\textbf{ACIC2016}. 
The 2016 Atlantic Causal Inference Challenge (ACIC 2016\footnote{ACIC 2016 datasets are available at: \url{https://github.com/vdorie/aciccomp/tree/master/2016}.}) \cite{dorie2019automated} holds the causal inference data analysis challenge, which creates 4,802 units, with 858 in the treated group and 3,944 in the control group.
The two expected potential outcomes are $m_0$ for control outcomes and $m_1$ for treated outcomes. 
The covariates are real-world data from the full Infant Health and Development Program dataset, which consists of 58 variables ($C \in \mathbb{R}^{58}$).
In the above ACIC study, we use the large control group as $\mathcal{D}^{(e_0)}$, and the small treated group as $\mathcal{D}^{(e_1)}$. We select $\text{d}_\text{x} =12$ continuous variables from the ACIC as observed covariates and use the expected potential outcomes $m_0$ and $m_1$ as unmeasured confounding effects from the remaining $\text{d}_\text{u} = 46$ unmeasured variables. The encouraged treatments are from $T^{(e_k)} = | \sum_{i=1}^{\text{d}_\text{x}-1} [\phi_i^{(e_k)}X_i^{(e_k)} X_{i+1}^{(e_k)}] + \sum_{i=1}^{\text{d}_\text{x}} [\phi_{i}^{(e_k)}X_{i}^{(e_k)}] + 4m_0^{(e_k)} |$, and outcomes are determined by $Y^{(e_k)}_{\textsc{IHDP}} = T^{(e_k)} \times (0.5+X_0^{(e_k)}) + \sum_{i=1}^{\text{d}_\text{x}-1} [\psi_i X_i^{(e_k)} X_{i+1}^{(e_k)}] + \sum_{i=1}^{\text{d}_\text{x}} [\psi_i X_{i}^{(e_k)}] + 4m_1^{(e_k)}$. From the control group $\mathcal{D}^{(e_0)}$, We split 480 samples for validation data and 480 for pre-testing data, leaving $n_0 = \text{2,984}$ samples as encouragement data, while maintaining $n_1 = \text{858}$ in the treated group. We then replicate the pre-testing data 20 times, creating 9,600 samples with random treatments $do(t) \sim \text{Unif}[0, 1]$, and regenerate outcomes $Y(do(t),X)$ for testing.

\begin{table*}[t]
{
\caption{Results ($\text{mean}_{\pm \text{std}}$) on Complex Simulation with Additional \textsc{Poly}, \textsc{Abs} and \textsc{Sin} Terms.}
\label{tab:nonlinear}
\begin{center}
\setlength{\tabcolsep}{2mm}
{
\begin{tabular}{c|cc|cc|cc}
\toprule
 & \multicolumn{2}{c|}{\bf \textsc{Poly}} & \multicolumn{2}{c|}{\bf \textsc{Abs}}
& \multicolumn{2}{c}{\bf \textsc{Sin}} \\
\midrule
Methods & \bf
$\epsilon_{\text{CFE}}$ & \bf $\epsilon_{\text{PEHE}}$ & \bf $\epsilon_{\text{CFE}}$ & \bf $\epsilon_{\text{PEHE}}$ & \bf $\epsilon_{\text{CFE}}$ & \bf $\epsilon_{\text{PEHE}}$ \\
\midrule

\bf KernelIV & $ 25.31_{\pm 1.906} $ & $ 0.585_{\pm 0.040} $  & $ 21.44_{\pm 2.115} $ & $ 0.593_{\pm 0.106} $  & $ 19.67_{\pm 2.625} $ & $ 0.644_{\pm 0.189} $ \\
\bf AGMM & \underline{$ 6.502_{\pm 0.954} $} & \underline{$ 0.294_{\pm 0.022} $}  & \underline{$ 6.133_{\pm 0.896} $} & \underline{$ 0.256_{\pm 0.024} $}  & \underline{$ 6.674_{\pm 1.018} $} & \underline{$ 0.281_{\pm 0.030} $} \\
\bf CBIV & $ 11.64_{\pm 4.388} $ & $ 0.506_{\pm 0.102} $  & $ 10.01_{\pm 2.869} $ & $ 0.499_{\pm 0.123} $  & $ 10.03_{\pm 2.921} $ & $ 0.538_{\pm 0.142} $ \\

\midrule

\bf VCNet & $ 7.628_{\pm 0.605} $ & $ 0.301_{\pm 0.035} $  & $ 7.083_{\pm 0.532} $ & $ 0.283_{\pm 0.044} $  & $ 8.384_{\pm 0.317} $ & $ 0.341_{\pm 0.023} $ \\
\bf CEVAE & $ 11.93_{\pm 1.212} $ & $ 0.559_{\pm 0.032} $  & $ 10.69_{\pm 0.770} $ & $ 0.526_{\pm 0.050} $  & $ 10.74_{\pm 0.417} $ & $ 0.556_{\pm 0.035} $ \\
\bf KerIRM & $ 19.52_{\pm 2.365} $ & $ 0.506_{\pm 0.057} $  & $ 16.35_{\pm 3.090} $ & $ 0.488_{\pm 0.061} $  & $ 15.24_{\pm 2.957} $ & $ 0.495_{\pm 0.080} $ \\

\midrule

\bf VANILLA & 
$ 8.362_{\pm 0.972} $ & $ 0.352_{\pm 0.080} $ 
& $ 8.425_{\pm 0.929} $ & $ 0.343_{\pm 0.081} $ 
& $ 7.708_{\pm 1.171} $ & $ 0.348_{\pm 0.036} $ 
\\

\bf EnCounteR  & 
$ \mathbf{ 5.294_{\pm 0.434} } $ & $ \mathbf{ 0.214_{\pm 0.033} } $ 
& $ \mathbf{ 5.029_{\pm 0.446} } $ & $ \mathbf{ 0.226_{\pm 0.026} } $ 
& $ \mathbf{ 4.840_{\pm 0.616} } $ & $ \mathbf{ 0.222_{\pm 0.033} } $ 
\\
 
\bottomrule
\end{tabular}
}
\end{center}
}

\end{table*}

\paragraph{Main Results} We compare our method with baselines for estimating the counterfactual outcomes and CATE on the above datasets, each with 10 replications. The mean and standard deviation of $\varepsilon_{\text{CFR}}$ and $\varepsilon_{\text{PEHE}}$ are shown in Table \ref{tab:main}, and the optimal and second-optimal performance are \textbf{bold} and \underline{underlined}, respectively. 
First, from Table \ref{tab:main}, we can find that IV-based methods are limited in fully capturing exogenous variations in continuous treatments due to discrete encouragements, failing to precisely estimate causal effects as the exogeneity is insufficient for confounding effects. Second, covariate-based methods such as CFRNet, DRCFR, and VCNet also underperform because the unconfoundedness assumption is violated; as seen from the results, CEVAE and TEDVAE are prone to overfit; and methods including KerIRM and KerHRM fail to deal with unmeasured confounding and observed variables' entanglements. 
Third, the proposed EnCounteR outperforms all baselines across various datasets, achieving state-of-the-art performance. Compared to the \underline{second-optimal} method, our EnCounteR on Simulation, IHDP, and ACIC datasets further reduce $\epsilon_{\text{CFE}}$ by 19\%, 10\%, and 31\%, and $\epsilon_{\text{PEHE}}$ by 23\%, 18\%, and 29\%, respectively. These results highlight the scalability of our method to complex data, demonstrating its potential for real applications.

\paragraph{The scalability of our EnCounteR across varying non-linear}
To simulate more complex real-world applications, we conduct additional experiments incorporating three non-linear terms into the outcome function (described as \textsc{Poly}, \textsc{Abs}, and \textsc{Sin}). 
Observing that most traditional methods underperform in the \textsc{Mult} experiments (Table \ref{tab:main} and Figure \ref{fig:param}), we select three methods each from IV-based and covariate-based approaches for further experimental performance comparison. From the IV-based methods, we choose KernelIV \cite{singh2019kernel}, AGMM \cite{lewis2020agmm}, and CBIV \cite{wu2022instrumental}, and from the covariate-based methods, we select VCNet \cite{nie2021vcnet}, CEVAE \cite{louizos2017causal}, and KerIRM \cite{arjovsky2019invariant}. The results in Table \ref{tab:nonlinear} reveal that most traditional methods have significant estimation errors, some even worse than the basic direct regression (VANILLA) method. Only AGMM and VCNet show improvements over the VANILLA method. Our EnCounteR, however, exhibits robust and outstanding performance on more complex datasets. Compared to the optimal-second AGMM algorithm, our EnCounteR further reduces the $\epsilon_{\text{CFE}}$ by 18\%, 18\%, and 27\% and the $\epsilon_{\text{PEHE}}$ by 27\%, 11\%, and 21\% on the Simulation, IHDP, and ACIC datasets, respectively. These results highlight the scalability of our method, demonstrating its potential for real-world applications.

\paragraph{The scalability of our EnCounteR across varying encouragements $K$ and sub-data volume $n_k$} 
Moreover, we study the scalability of our EnCounteR across varying encouragements $K$ and data volume $K \times n_k=600$ with $n_0=\text{2,000}$ on \textsc{Mult1}, \textsc{Mult2}, and \textsc{Mult3}. As shown in Figure \ref{fig:box}, more encouragements $K$ and smaller $n_k$ may result in larger errors and higher variance in EnCounteR on \textsc{Mult}s. In contrast, on \textsc{Mult1}, \textsc{Mult4}, and \textsc{Mult5}, with $n_k=\text{600}$ fixed, we find that as $K$ increases, the mean error of EnCounteR on \textsc{Mult}s remains unchanged, but the variance decreases. This suggests our EnCounteR depends mainly on the largest $n_k$ and encouragements $K$, as a consequence, applying single encouragement to more units would yield better estimation.

\begin{table*}[t]
{\small
\caption{Ablation Study of EnCounteR Framework on Simulation, IHDP, and ACIC Datasets. EnCounteR is composed by four core modules: (a) $\boldsymbol{\omega}$: Sample Reweighting Module in Eq. \eqref{eq:17}; (b) $\boldsymbol{\mathcal{L}}_{\boldsymbol{\mathcal{E}}}$: Encouragement-Independent Moments in Eq. \eqref{eq:18}; (c) $\boldsymbol{\mathcal{L}}_{\boldsymbol{X}}$: Covariate-Independent Moments in Eq. \eqref{eq:19}; (d) $\boldsymbol{\mathcal{L}}_{\boldsymbol{R}}$: Adversarial Representation-Independent Moments in Eq. \eqref{eq:20}.}
\label{tab:ablation}
\begin{center}
\setlength{\tabcolsep}{1mm}{
\begin{tabular}{c|cccc|cc|cc|cc}
\toprule
& \multicolumn{4}{c|}{\bf Modules} & \multicolumn{2}{c|}{\bf Simulation (\textsc{Mult})} & \multicolumn{2}{c|}{\bf IHDP}
& \multicolumn{2}{c}{\bf ACIC} \\
\midrule
\textbf{EnCounteR} & $+\boldsymbol{\omega}$  & $+\boldsymbol{\mathcal{L}}_{\boldsymbol{\mathcal{E}}}$ & \bf $+\boldsymbol{\mathcal{L}}_{\boldsymbol{X}}$  & $+\boldsymbol{\mathcal{L}}_{\boldsymbol{R}}$  & \bf $\epsilon_{\text{CFE}}$ & \bf $\epsilon_{\text{PEHE}}$ & \bf $\epsilon_{\text{CFE}}$ & \bf $\epsilon_{\text{PEHE}}$ & \bf $\epsilon_{\text{CFE}}$ & \bf $\epsilon_{\text{PEHE}}$ \\
\midrule

$\boldsymbol{\mathcal{L}}_\textbf{REG}$ & 
& & & & 
$ 7.512_{\pm 1.048} $ & $ 0.348_{\pm 0.067} $ 
& $ 2.068_{\pm 1.917} $ & $ 0.510_{\pm 0.323} $ 
& $ 19.66_{\pm 14.98} $ & $ 0.656_{\pm 0.317} $ 
\\

$\boldsymbol{\mathcal{L}}_\textbf{REG}$ & 
\checkmark & & & & 
$ 7.311_{\pm 1.020} $ & $ 0.344_{\pm 0.056} $ 
& $ 1.565_{\pm 0.984} $ & $ 0.452_{\pm 0.208} $ 
& $ 14.78_{\pm 6.163} $ & $ 0.587_{\pm 0.197} $ 
\\

$\boldsymbol{\mathcal{L}}_\textbf{REG}$ & 
\checkmark & \checkmark & &  &
$ 7.841_{\pm 1.487} $ & $ 0.343_{\pm 0.082} $ 
& $ 1.306_{\pm 0.521} $ & $ 0.378_{\pm 0.124} $ 
& $ 12.31_{\pm 5.429} $ & $ 0.457_{\pm 0.126} $ 
\\

$\boldsymbol{\mathcal{L}}_\textbf{REG}$ & 
\checkmark & \checkmark & \checkmark & & 
$ 5.191_{\pm 0.533} $ & $ 0.224_{\pm 0.036} $ 
& $ 0.665_{\pm 0.213} $ & $ 0.215_{\pm 0.066} $ 
& $ 12.58_{\pm 5.421} $ & $ 0.476_{\pm 0.121} $ 
\\

$\boldsymbol{\mathcal{L}}_\textbf{REG}$ & 
\checkmark & \checkmark & & \checkmark & 
$ 5.733_{\pm 0.816} $ & $ 0.259_{\pm 0.036} $ 
& $ 0.710_{\pm 0.156} $ & $ 0.229_{\pm 0.024} $ 
& \underline{$ 5.689_{\pm 0.669} $} & \underline{$ 0.204_{\pm 0.022} $} 
\\

$\boldsymbol{\mathcal{L}}_\textbf{REG}$ & 
& \checkmark & \checkmark & \checkmark & 
\underline{$ 4.847_{\pm 0.607} $} & \underline{$ 0.220_{\pm 0.037} $} 
& \underline{$ 0.641_{\pm 0.203} $} & \underline{$ 0.199_{\pm 0.025} $} 
& $ 6.067_{\pm 0.927} $ & $ 0.218_{\pm 0.036} $ 
\\

$\boldsymbol{\mathcal{L}}_\textbf{REG}$ & 
\checkmark & \checkmark & \checkmark & \checkmark & 
$ \mathbf{ 4.816_{\pm 0.609} } $ & $ \mathbf{ 0.210_{\pm 0.026} } $ 
& $ \mathbf{ 0.582_{\pm 0.130} } $ & $ \mathbf{ 0.188_{\pm 0.021} } $ 
& $ \mathbf{ 5.751_{\pm 0.606} } $ & $ \mathbf{ 0.186_{\pm 0.038} } $ 
\\
 
\bottomrule
\end{tabular}
}
\end{center}
}
\end{table*}

\begin{figure*}[th]
\begin{center}
\includegraphics[width=1.00\linewidth]{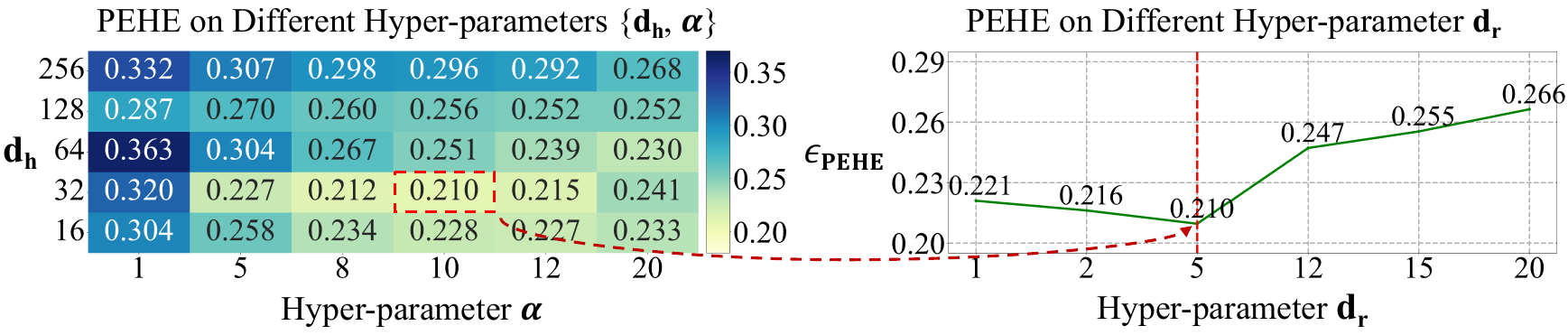}
\end{center}
\caption{Hyper-Parameter Optimization: The minimum regression error on the validation data implies the optimal hyper-parameters. The optimal hyper-parameters are $\text{d}_\text{h}=32, \text{d}_\text{r}=5, \alpha=10$ for \textsc{Mult}. }
\label{fig:param}
\end{figure*}

\begin{table*}[th]
{
\caption{Optimal Parameters on \textsc{Mult}, \textsc{Poly}, \textsc{Abs}, \textsc{Sin}, \textsc{IHDP} and \textsc{ACIC} Datasets.}
\label{tab:params}
\begin{center}
\setlength{\tabcolsep}{3mm}
{
\begin{tabular}{c|c|c|c|c|c|c}
\toprule
Params & 
\textsc{Mult} &
\textsc{Poly} &
\textsc{Abs} &
\textsc{Sin} &
\textsc{IHDP} &
\textsc{ACIC} 
\\

\midrule

$\mathbf{ \alpha }$ & 
10 &
8 &
10 &
8 &
5 &
8 
\\

$ \mathbf{ d_h }$ & 
32 &
128 &
128 &
32 &
128 &
32 
\\

$ \mathbf{ d_r }$ & 
5 &
2 &
5 &
12 &
15 &
12 
\\

\midrule
\bf $\epsilon_{\text{PEHE}}$ & 
$ \mathbf{ 0.210_{\pm 0.026} } $  &
$ \mathbf{ 0.214_{\pm 0.033} } $  &
$ \mathbf{ 0.226_{\pm 0.026} } $ &
$ \mathbf{ 0.222_{\pm 0.033} } $  &
$ \mathbf{ 0.188_{\pm 0.021} } $ &
$ \mathbf{ 0.186_{\pm 0.038} } $
\\
 
\bottomrule
\end{tabular}
}
\end{center}
}
\end{table*}

\subsection{Ablation Studies} 
EnCounteR is composed by four core modules: (a) $\boldsymbol{\omega}$: Sample Reweighting Module in Eq. \eqref{eq:17}; (b) $\boldsymbol{\mathcal{L}}_{\boldsymbol{\mathcal{E}}}$: Encouragement-Independent Moments in Eq. \eqref{eq:18}; (c) $\boldsymbol{\mathcal{L}}_{\boldsymbol{X}}$: Covariate-Independent Moments in Eq. \eqref{eq:19}; (d) $\boldsymbol{\mathcal{L}}_{\boldsymbol{R}}$: Adversarial Representation-Independent Moments in Eq. \eqref{eq:20}.
Table \ref{tab:ablation} reports the effects of each module of the EnCounteR by conducting ablation experiments on Simulation, IHDP and ACIC datasets.
From Tables \ref{tab:main} and Table \ref{tab:ablation}, we can draw the following conclusions: (I) Each component in our EnCounteR is essential, since missing any one of them would confuse the encouragement learning and damage the performance of potential outcome prediction and conditional average treatment estimation on three datasets. (II) When all components are fully utilized in EnCounteR, our method achieves optimal performance in causal effect estimation. The results demonstrate that each component of EnCounteR is crucial for estimating counterfactual outcomes and CATE.

\subsection{The Optimization of Hyper-Parameters}
\label{app:parameters}

In this paper, we adopt the minimum counterfactual regression error $\varepsilon_{\text{CFR}}$ on the validation data to determine the optimal hyper-parameters $\{\text{d}_\text{h},\text{d}_\text{r},\alpha\}$. Our approach follows this strategy: firstly, we search for $\text{d}_\text{h} \in \{16,32,64,128,256\}$ and $\alpha \in \{1,2,5,12,15,20\}$, while fixing $\text{d}_\text{r}=\text{d}_\text{x}$, corresponding to the minimum validation error $\varepsilon_{\text{CFR}}$. Then, fixing the optimal $\text{d}_\text{h}$ and $\alpha$, we search for $\text{d}_\text{r} \in \{1,5,8,10,12,20\}$ corresponding to the minimum validation error $\varepsilon_{\text{CFR}}$ again. 
Taking the main experiment \textsc{Mult} as an example, as depicted in Figure \ref{fig:param}, we determine the hyper-parameters that correspond to the smallest $\varepsilon_{\text{CFR}}$ on the validation, which also indicates the smallest $\varepsilon_{\text{PEHE}}$ on \textsc{Mult}. The optimal hyper-parameters are $\text{d}_\text{h}=32, \text{d}_\text{r}=5, \alpha=10$ for \textsc{Mult}. Table \ref{tab:params} shows the optimal hyper-parameters for each dataset.

\section{Conclusion}
Despite the growing body of literature on encouragement designs (EDs) for estimating causal effects, real-world applications often face challenges such as incomplete randomization, limited data, and fewer encouragements than continuous treatments. To address these challenges, we introduce a generalized instrumental variables estimator called \textbf{En}couragement-based \textbf{Counte}rfactual \textbf{R}egression (\textbf{EnCounteR}), which provides identifiability guarantees and efficient methods for estimating Conditional Average Treatment Effects (CATE) in the context of diverse positive encouragement experiments. By designing encouragements that only motivate treatment choices without affecting outcomes, EnCounteR achieves precise treatment effect estimation with reduced variance, applicable to both discrete and continuous treatment settings.

This method is particularly beneficial for analysts comparing different treatments in decision-making scenarios. For instance, it can assist educators in evaluating the causal effects of various teaching strategies to develop personalized educational plans for students. Similarly, it can support the tailored implementation of healthcare interventions and the adaptation of policies to local contexts. A limitation of EnCounteR is that it requires at least two distinct encouragement experiments with variations in treatment. However, we suggest focusing on positive encouragements, such as promoting smoking cessation advice among physicians, which is widely adopted in practice.

\end{document}